\documentclass[lettersize,journal]{IEEEtran}
\usepackage{amsmath,amsfonts}
\usepackage{array}
\usepackage[caption=false,font=normalsize,labelfont=sf,textfont=sf]{subfig}
\usepackage{textcomp}
\usepackage{stfloats}
\usepackage{url}
\usepackage{verbatim}
\usepackage{graphicx}
\usepackage{cite}

\usepackage{booktabs}
\usepackage{lipsum}
\usepackage{multirow}
\usepackage{rotating}
\usepackage{adjustbox}
\usepackage{tikz}
\usetikzlibrary{matrix,calc}

\usepackage{amsmath}
\usepackage[ruled, vlined, linesnumbered]{algorithm2e}

\usepackage{amsthm}
\theoremstyle{definition}
\newtheorem{definition}{Definition}

\newtheorem{theorem}{Theorem}

\begin{document}

\title{Prasatul Matrix: A Direct Comparison Approach for Analyzing Evolutionary Optimization Algorithms}

\author{Anupam~Biswas,~\IEEEmembership{Member,~IEEE}
\thanks{A. Biswas is with the Department of Computer Science and Engineering, National Institute of Technology Silchar, Cachar, Assam, 788010, India.\\
E-mail: anupam@cse.nits.ac.in}
\thanks{Manuscript received Month xx, 20yy; revised Month xx, 20yy.}}

\markboth{JOURNAL OF LATEX CLASS FILES, VOL. XX, NO. XX, MONTH YYYY}%
{Biswas : A Direct Comparison Approach for Analyzing Evolutionary Optimization Algorithms}


\maketitle

\begin{abstract}
The performance of individual evolutionary optimization algorithms are mostly measured in terms of statistics such as mean, median and standard deviation etc., computed over the best solutions obtained with few trails of the algorithm. To compare the performance of two algorithms, the values of these statistics are compared instead of comparing the solutions directly. This kind of comparison lacks direct comparison of solutions obtained with different algorithms. For instance, the comparison of best solutions (or worst solution) of two algorithms simply not possible. Moreover, ranking of algorithms is mostly done in terms of solution quality only, despite the fact that the convergence of algorithm is also an important factor. In this paper, a direct comparison approach is proposed to analyze the performance of evolutionary optimization algorithms. A direct comparison matrix called \emph{Prasatul Matrix} is prepared, which accounts direct comparison outcome of best solutions obtained with two algorithms for a specific number of trials. Five different performance measures are designed based on the prasatul matrix to evaluate the performance of algorithms in terms of Optimality and Comparability of solutions. These scores are utilized to develop a score-driven approach for comparing performance of multiple algorithms as well as for ranking both in the grounds of solution quality and convergence analysis. Proposed approach is analyzed with six evolutionary optimization algorithms on 25 benchmark functions. A non-parametric statistical analysis namely Wilcoxon paired sum-rank test is also performed to verify the outcomes of proposed direct comparison approach.

\end{abstract}


\section{Introduction}
\IEEEPARstart{E}{volutionary} Optimization Algorithms (EOAs) as well as their applications have grown manifold in last decades. Thus, it is necessary to develop effective evaluation methodology that helps in better understanding the performance of the algorithm. EOAs are evaluated primarily from the perspective of solution quality and convergence. Solution quality is measured on the basis of solutions obtained in different trials of the algorithm. On the other hand, empirical convergence analysis is mostly done on the basis of solutions obtained in different iterations for a single trial of the algorithm. Majority of evaluation techniques focus on solution quality of EOAs, which include non-parametric approaches~\cite{derrac2011practical}, parametric approaches~\cite{czarn2004statistical,liu2022cooperative}, statistical tests~\cite{carrasco2020recent}, Bootstrapping~\cite{carrano2008enhanced,nijssen2003analysis}, Drift analysis~\cite{he2001drift}, Exploratory Landscape Analysis (ELA)~\cite{mersmann2010benchmarking}, theoretic analysis~\cite{lockett2013measure,Xiaoliang2021} and visual analysis approaches~\cite{lutton2011visual,biswas2014visual,biswas2017analyzing,biswas2017regression,Lin2022} etc. Whereas, convergence analysis in general is done through visual inspection of graphical presentation of solutions obtained in different iterations sequentially~\cite{mirjalili2015ant,mirjalili2016whale,mirjalili2016sca,biswas2018particle,biswas2016atom,sarkar2022comparative,Wang2022}.

Although many approaches have been developed, parametric and non-parametric approaches are widely used for analyzing solution quality of EOAs. Most of these approaches incorporate results obtained on benchmark functions or some specific problems. Basic statistical measures such as standard deviation, mean, median, maximum, minimum are estimated for comparing performance of EOAs. However, comparison of statistics obtained for two algorithms is an indirect approach as the solutions obtained with two algorithms are not compared directly. Most importantly, same values of statistics do not imply that the solutions of two algorithms are exactly the same. For instance, two algorithms may have the same mean values, but this certainly does not imply performance of both the algorithms are same as well. This is the reason why along with mean, standard deviation and other statistics are often considered and interpreted together. However, observing multiple statistics to draw conclusion poses added difficulty in the performance comparison process of EOAs.

In this paper,  a direct comparison approach is proposed to analyze performance of EOAs, where solutions obtained with two algorithms are compared directly. A direct comparison matrix called \emph{Prasatul Matrix} is prepared based on the optimality and comparability level of the solutions obtained with two different algorithms. The prasatul matrix is used to design five different performance measures as well as ranking schemes for EOAs. The key features of the proposed approach are as follows: 
\begin{itemize}
    \item Unlike comparing statistics such as mean and standard deviation, the solutions obtained with two algorithms are directly compared and recorded in the prasatul matrix.
    \item The five measures (Section~\ref{sec2.4}) which are designed based on prasatul matrix are equally capable of comparing EOAs both in the grounds of solution quality and convergence.
    \item Unlike comparing scores obtained with two algorithms, here the score itself gives the comparative outcome of solutions obtained with two EOAs (Section~\ref{sec3.1}). Interpretation of proposed measures are done in pairs of two EOAs.
    \item Score-driven ranking schemes (Section~\ref{sec3.2}) designed based on prasatul matrix are also capable of ranking EOAs both in the grounds of solution quality and convergence.
    \item To best of our knowledge, for the very first time this work introduces score-driven comparison as well as ranking for analyzing convergence of EOAs.
\end{itemize}

Rest of the paper is organized as follows. Section~II describes the proposed direct comparison approach covering preliminary definitions, the prasatul matrix, different measures designed and the algorithm for generating prasatul matrix for a pair of two EOAs. Section~III details about how the newly designed measures can be utilized for direct comparison of multiple algorithms and ranking both in terms of quality and convergence. Section~IV details about the experimental setup. Section~V presents the results on several benchmark functions. Finally, concluded in Section~VI.

\hfill

\section{Direct Comparison Approach}
\subsection{Preliminary Definitions}
The direct comparison approach involves two EOAs in the process. The role of the participating optimization algorithms in the direct comparison approach are defined as follows:
\begin{definition}[Primary Algorithm ($A_p$)]
The algorithm whose performance is to be evaluated in comparison to other algorithm is referred as primary algorithm.
\end{definition}

\begin{definition}[Alternative Algorithm ($A_q$)]
The algorithms with whom the primary algorithm is to be compared is referred as alternative algorithm.
\end{definition}

This work considers only single objective minimization problems. Therefore, all definitions are given in the context of minimization problems. Let us consider $n$ best solutions obtained with each of the two algorithms $A_p$ and $A_q$  for $n$ trials on the same minimization problem (i.e. objective function) and arranged in a non-decreasing ordered list as per the quality of the solutions as $P=[p_1, p_2, p_3, ..., p_n]$ and $Q=[q_1, q_2, q_3, ..., q_n]$ respectively. Let $U=[u_1, u_2, u_3, ..., u_{2n}]$ be the universe comprising all the best solutions in $P$ and $Q$ obtained with algorithm $A_p$  and $A_q$ respectively, which is also a non-decreasing ordered list as per the quality of the solutions.  Therefore, $u_1$ and $u_{2n}$ are termed as the \textit{Universe Best ($U_b$)} and  \textit{Universe Worst ($U_w$)} respectively. 




\begin{definition}[Universe Mean ($U_\sigma$)]
 Mean or average of all best solutions in the Universe is defined as:
\begin{equation}
    U_\sigma = \frac{\sum_{\forall u_i \in U} u_i}{|U|}
\end{equation}
\end{definition}




\begin{definition}[Optima ($O$)]
It is the optimal value of the finite search space or solution space of the objective function $f(x)$ considered for both algorithms $A_p$ and $A_q$. The optima for a minimization problem having finite search space is defined as:

\begin{equation}
  O = \arg \min_{x} f(x).
\end{equation}
\end{definition}

\subsection{Comparability and Optimality Levels}

The comparability level of a solution is determined by simply comparing with another solution. The three levels of comparability are defined as follows:

\begin{definition}[Comparability Level 1]
 If $i^{th}$ solution $p_i$ obtained with algorithm $A_p$ is better than the solution $q_i$ of algorithm $A_q$ then $p_i$ is termed as in Comparability Level 1. Thus, for all the solutions in $P$ in comparison to the solutions in $Q$, the Comparability Level 1 set ($L_1^c$) is defined as follows: 
\begin{equation}
    L_1^c=\{p_i | p_i<q_i, \forall (p_i,q_i)\}
\end{equation}
\end{definition}

\begin{definition}[Comparability Level 2]
 If $i^{th}$ solution $p_i$ obtained with algorithm $A_p$ is equal to the solution $q_i$ of algorithm $A_q$ then $p_i$ is termed as in Comparability Level 2. Thus, for all the solutions in $P$ in comparison to the solutions in $Q$,  the Comparability Level 2 set ($L_2^c$) is defined as follows: 
\begin{equation}
    L_2^c=\{p_i | p_i=q_i, \forall (p_i,q_i)\}
\end{equation}
\end{definition}

\begin{definition}[Comparability Level 3]
 If $i^{th}$ solution $p_i$ obtained with algorithm $A_p$ is greater than the solution $q_i$ of algorithm $A_q$ then $p_i$ is termed as in Comparability Level 3. Thus, for all the solutions in $P$ in comparison to the solutions in $Q$, the Comparability Level 3 set ($L_3^c$) is defined as follows: 
\begin{equation}
    L_3^c=\{p_i | p_i>q_i, \forall (p_i,q_i)\}
\end{equation}
\end{definition}

\begin{figure}
    \centering
    \includegraphics[width=\columnwidth]{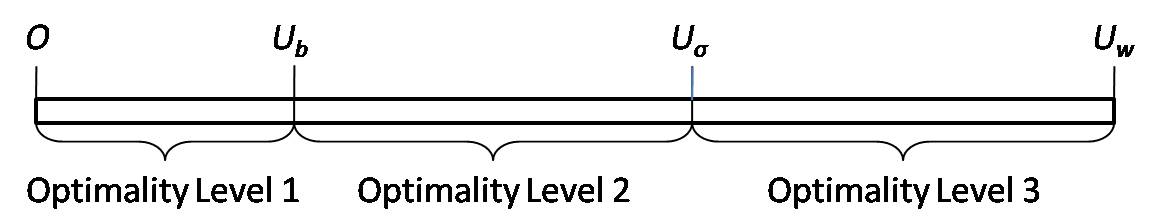}
    \caption{Pictorial depiction of Optimality Levels based on optima and universe}
    \label{fig:OpLevel}
\end{figure}

The optimality level of a solution is defined in terms of its distance from the actual optima of the given problem (minimization problem) as shown in the Fig.~\ref{fig:OpLevel}. The three levels of optimality are defined as follows:
\begin{definition}[Optimality Level 1]
 If $i^{th}$ solution $p_i$ obtained with algorithm $A_p$ lies in between the optima ($O$) and universe best ($U_b$) including $O$ and $U_b$ then $p_i$ is termed as in Optimality Level 1. Thus, for all the solutions in $P$, the Optimality Level 1 set ($L_1^o$) is defined as follows:
\begin{equation}
    L_1^o=\{p_i|p_i\in[O, U_b]\}
\end{equation}
\end{definition}

\begin{definition}[Optimality Level 2]
 If $i^{th}$ solution $p_i$ obtained with algorithm $A_p$ lies in between the universe best ($U_b$) and universe mean ($U_\sigma$) including $U_\sigma$ then $p_i$ is termed as in Optimality Level 2. Thus, for all the solutions in $P$, the Optimality Level 2 set ($L_2^o$) is defined as follows:
\begin{equation}
    L_2^o=\{p_i|p_i\in(U_b,U_\sigma]\}
\end{equation}
\end{definition}

\begin{definition}[Optimality Level 3]
If $i^{th}$ solution $p_i$ obtained with algorithm $A_p$ lies in between the universe best ($U_\sigma$) and universe worst ($U_w$) including $U_w$ then $p_i$ is termed as in Optimality Level 3. Thus, for all the solutions in $P$, the Optimality Level 3 set ($L_3^o$) is defined as follows:
\begin{equation}
    L_3^o=\{p_i|p_i\in(U_\sigma,U_w]\}
\end{equation}
\end{definition}

\subsection{Prasatul Matrix}
\label{sec2.3}
The prasatul matrix $\mathcal{L}$ is prepared by directly comparing the best solutions obtained with the primary algorithm $A_p$ and the best solutions obtained with the alternative algorithm $A_q$ for $n$ trials. The prasatul matrix has three levels of abstractions both in rows and columns, which are defined on the basis of comparability and optimality w.r.t. alternatives and optima respectively. Here, alternatives imply any algorithm $A_q$ whose performance is to be compared with the algorithm $A_p$. A typical prasatul matrix with different abstraction level w.r.t. alternatives and optima is presented in Fig.~\ref{fig:compmat}. 

Comparison of solutions of $A_p$ with the alternative $A_q$ is enumerated in terms of Comparability Level 1, 2 and 3, which corresponds to three levels of abstraction: \emph{Win}, \emph{Tie} and \emph{Lose} respectively. These are presented in the rows of the prasatul matrix, which indicates how good or bad the solution is w.r.t. alternatives. The abstraction Win means solution of $A_p$ is better than $A_q$ and corresponding row is count of such wining solutions. The abstraction Tie means the solution of $A_p$ is same as $A_q$ and corresponding row is count of such solutions. The abstraction Lose means the solution of $A_p$ is worse than $A_q$ and corresponding row counts such losing solutions.


On the other hand, comparison of best solutions of $A_p$ with the optima is enumerated in terms of Optimality Level 1, 2 and 3, which corresponds to three levels of abstraction: \emph{Best}, \emph{Average} and \emph{Worst}. These are presented in the columns of the prasatul matrix, which indicates how good or bad the solution is w.r.t. optima. The abstraction \emph{Best} means the solution of $A_p$ is at most $U_b$ away from the optima  and corresponding column is the count of such solutions. The abstraction \emph{Average} means the solution of $A_p$ is at most $U_\sigma$ and more than $U_b$ away from the optima and corresponding column is count of such solutions. The abstraction \emph{Worst} means the solution of $A_p$ is more than $U_\sigma$ away from the optima and corresponding column is count of such solutions.

Considering both comparability and optimality levels, if the abstraction levels Win and Best are taken together then it is the count of solutions that are in both $L_1^c$ and  $L_1^o$, i.e. $|L_1^c\cap L_1^o|$. Similarly, if the abstraction levels Win and Average are taken together than it will be $|L_1^c\cap L_2^o|$ and so on. Therefore, the elements of the prasatul matrix $\mathcal{L}$ are defined as follows:

\begin{equation}
\mathcal{L}_{ij}= |L_i^c\cap L_j^o|, \forall i,j
\end{equation}

\subsection{Measures and Interpretations}
\label{sec2.4}
Interpretation of performance of algorithms are done in terms of Optimality of the solutions and Comparability of solutions. However, the quality of the solutions are determined on the basis of direct comparison and overall 
comparison, which are referred as D-scores and K-scores respectively. For direct comparison, two D-scores are defined: Direct Optimality (DO) and Direct Comparability (DC). For overall comparison, three K-scores are defined: Overall Optimality (KO), Overall Comparability (KC) and Overall Together (KT).

\begin{figure}[t]
    \centering
\resizebox{\columnwidth}{!}{    
    \begin{tabular}{l|l|c|c|c|c}
\multicolumn{2}{c}{}&\multicolumn{3}{c}{Optima}&\\
\cline{3-5}
\multicolumn{2}{c|}{}&$Best$&$Average$&$Worst$&\multicolumn{1}{c}{Total}\\
\cline{2-5}
\multirow{3}{*}{\begin{turn}{90}Alternative\end{turn}}& $Win$ & $a$ & $b$ & $c$ & $a+b+c$\\
\cline{2-5}
& $Tie$ & $d$ & $e$ & $f$ & $d+e+f$\\
\cline{2-5}
& $Lose$ & $g$ & $h$ & $i$ & $g+h+i$\\
\cline{2-5}
\multicolumn{1}{c}{} & \multicolumn{1}{c}{Total} & \multicolumn{1}{c}{$a+d+g$} & \multicolumn{    1}{c}{$b+e+h$} & \multicolumn{    1}{c}{$c+f+i$} & \multicolumn{1}{c}{$n$}\\
\end{tabular}
}
    \caption{Prasatul Matrix ($\mathcal{L}$) with various abstraction levels}
    \label{fig:compmat}
\end{figure}

\subsubsection{D-scores}
Direct comparison of two algorithms are defined in terms of Optimality and Comparability, which are referred as D-scores. As discussed above, the Optimality is measured on three levels of abstractions: \emph{Best}, \emph{Average} and \emph{Worst}. The abstraction Best i.e. Level 1 Optimality is defined as:
\begin{equation}
    O_1=\frac{\mathcal{L}_{11}}{1+\sum_{i=1}^3{\mathcal{L}_{i1}}}
\end{equation}
where, high $O_1$ value is interpreted as best solutions of algorithm $A_p$ are better than best solutions of algorithm $A_q$. Likewise, Level 2 and Level 3 Optimality are defined as:

\begin{equation}
    O_2=\frac{\mathcal{L}_{12}}{1+\sum_{i=1}^3{\mathcal{L}_{i2}}}
\end{equation}
and
\begin{equation}
    O_3=\frac{\mathcal{L}_{13}}{1+\sum_{i=1}^3{\mathcal{L}_{i3}}}
\end{equation}
respectively. Here, high $O_2$ and $O_3$ values are interpreted as
average solutions of algorithm $A_p$ are better than average solutions of algorithm $A_q$ and worst solutions of algorithm $A_p$ are better than worst solutions of algorithm $A_q$ respectively. Thus, Direct Optimality (DO) is measured by combining weighted comparative Optimality of both algorithms in all three levels of abstractions. 
\begin{definition}[Direct Optimality (DO)] 
The Direct Optimality of algorithm $A_p$ in comparison to algorithm $A_q$ is defined as:
\begin{equation}
    DO=O_1+0.5*O_2-O_3
\end{equation}
where, high DO value indicates algorithm $A_p$ is better than algorithm $A_q$ in terms of Optimality.
\end{definition}

Similarly, Direct Comparability is also measured on three levels of abstractions: \emph{Win}, \emph{Tie} and \emph{Lose}. The abstraction Win i.e. Level 1 Comparability is defined as:

\begin{equation}
    C_1=\frac{\mathcal{L}_{11}}{1+\sum_{i=1}^3{\mathcal{L}_{1i}}}
\end{equation}
where, high $C_1$ values indicate the solutions of algorithm $A_p$ that are better than algorithm $A_q$ are the best solutions of $A_p$ in terms of Optimality. Likewise, the abstraction \emph{Tie} and \emph{Lose} i.e. Level 2 and Level 3 Comparability are defined as:
\begin{equation}
    C_2=\frac{\mathcal{L}_{21}}{1+\sum_{i=1}^3{\mathcal{L}_{2i}}}
\end{equation} 
and
\begin{equation}
    C_3=\frac{\mathcal{L}_{31}}{1+\sum_{i=1}^3{\mathcal{L}_{3i}}}
\end{equation}
respectively. Here, high $C_2$ values indicate that solutions of algorithm $A_p$ that are same as algorithm $A_q$ are the best solutions of $A_p$ in terms of Optimality. However, high $C_3$ values indicate that the solutions of algorithm $A_p$ that are worst in comparison to algorithm $A_q$ are the best solutions of $A_p$ in terms of optimality. The Direct Comparability (DC) is measured by combining weighted Comparability of both the algorithms in all three levels of abstractions. 
\begin{definition}[Direct Comparability (DC)] 
The Direct Comparability of algorithm $A_p$ in comparison to algorithm $A_q$ is defined as:
\begin{equation}
    DC=C_1+0.5*C_2-C_3
\end{equation}
where, high DC value indicates algorithm $A_p$ is better than algorithm $A_q$ in terms of Optimality.
\end{definition}

\subsubsection{K-Scores}
Unlike the D-Scores, where Optimality and Comparability are measured in the context of either row or columns to cover only one specific abstraction level, K-Scores consider all comparison outcomes into account. Therefore, K-Scores are considered overall performance both in the context of Optimality and Comparability. However, overall Optimality does not consider accounting comparability outcomes and overall Comparability does not consider accounting optimality outcomes. The overall Optimality in Level 1 i.e. in  abstraction level Best is defined as:  
\begin{equation}
    K_1^o=\frac{\sum_{i=1}^3{\mathcal{L}_{i1}}}{\sum{\mathcal{L}_{ii}}}
\end{equation}
where, high $K_1^o$ value indicates that overall solutions of comparing algorithm are Best in terms of optimality, irrespective of whether it loses or wins. The overall Optimality in Level 2 i.e. in  abstraction level Average is defined as:

\begin{equation}
     K_2^o=\frac{\sum_{i=1}^3{\mathcal{L}_{i2}}}{\sum{\mathcal{L}_{ii}}}
\end{equation}
where, high $K_2^o$ value indicates that overall solutions are Average in terms of optimality, irrespective of whether it loses or wins. The overall Optimality in Level 3 i.e. in  abstraction level Worst is defined as:

\begin{equation}
     K_3^o=\frac{\sum_{i=1}^3{\mathcal{L}_{i3}}}{\sum{\mathcal{L}_{ii}}}
\end{equation}
where, high $K_3^o$ value indicates that overall solutions are Worst in terms of optimality, irrespective of whether it loses or wins. The overall Optimality of comparing algorithm is measured by combining overall Optimality in all three levels of abstractions i.e. Best, Average and Worst.

\begin{definition}[Overall Optimality (KO)] 
The Overall Optimality of algorithm $A_p$ is defined as:
\begin{equation}
    KO=K_1^o+0.5*K_2^o-K_3^o
\end{equation}
\begin{equation}
    KO=\frac{\sum_{i=1}^3{\mathcal{L}_{i1}}+0.5*\sum_{i=1}^3{\mathcal{L}_{i2}}-\sum_{i=1}^3{\mathcal{L}_{i3}}}{n}
\end{equation}
where, high KO value indicates that the solutions of algorithm $A_p$ are highly Optimal.
\end{definition}

Similarly, Overall Comparability is also measured on three levels of abstractions: \emph{Win}, \emph{Tie} and \emph{Lose}. The overall Comparability in Level 1 i.e. in  abstraction level Win is defined as:

\begin{equation}
    K_1^c=\frac{\sum_{i=1}^3{\mathcal{L}_{1i}}}{\sum{\mathcal{L}_{ii}}}
\end{equation}
where, high $K_1^c$ value indicates that overall solutions are better compared to other algorithm,  as comparing algorithm mostly wins irrespective of whether it is Best, Average or Worst. The overall Comparability in Level 2 i.e. in  abstraction level Tie is defined as:

\begin{equation}
     K_2^c=\frac{\sum_{i=1}^3{\mathcal{L}_{2i}}}{\sum{\mathcal{L}_{ii}}}
\end{equation}
where, high $K_2^c$ value indicates that overall solutions are same as other algorithm,  as comparing algorithm mostly ties irrespective of whether it is Best, Average or Worst. The overall Comparability in Level 3 i.e. in  abstraction level Lose is defined as:
\begin{equation}
     K_3^c=\frac{\sum_{i=1}^3{\mathcal{L}_{3i}}}{\sum{\mathcal{L}_{ii}}}
\end{equation}
where, high $K_3^c$ value indicates that overall solutions are worse compared to other algorithm,  as comparing algorithm mostly ties irrespective of whether it is Best, Average or Worst. Now, the Overall Comparability of the comparing algorithm is measured by combining Overall Comparability in all three levels of abstractions i.e. Win, Tie and Lose.

\begin{definition}[Overall Comparability (KC)] 
The Overall Comparability of algorithm $A_p$ in comparison to algorithm $A_q$ is defined as:
\begin{equation}
    KC=K_1^c+0.5*K_2^c-K_3^c
\end{equation}
\begin{equation}
    KC=\frac{\sum_{i=1}^3{\mathcal{L}_{1i}}+0.5*\sum_{i=1}^3{\mathcal{L}_{2i}}-\sum_{i=1}^3{\mathcal{L}_{3i}}}{n}
\end{equation}
where, high KC value indicates that the solutions of algorithm $A_p$ are better than that of algorithm $A_q$.
\end{definition}

The  Level 1 and Level 2 measures are important for both Overall Optimality and Overall Comparability. To assert an algorithm to performs better than other algorithm, it should posses most of the comparison outcomes in abstraction levels Best \& Average as well as in Win \& Tie. Thus, to combine both Overall Optimality and Comparability together, elements of prasatul matrix, where abstraction levels Best \& Average and Win \& Tie are overlapping. 

\begin{definition}[Overall Together (KT)] 
The Overall Together measure of algorithm $A_p$ in comparison to algorithm $A_q$ is defined as:
\begin{equation}
   KT= \frac{a+b+d+e}{n}
\end{equation}
where, high KT value indicates that overall the algorithm $A_p$ are better than that of algorithm $A_q$ be it Optimality or Comparability.
\end{definition}

\begin{theorem}
D-scores and K-scores the ranges (-1, +1.5) and [-1, +1] respectively, but KT has the range [0, 1]. 
\end{theorem}
\begin{proof}
The proof is quite straight forward. Both D-scores and K-scores except KT has only one negative component, which can have maximum -1 value. However, for D-scores maximum value of the negative component will always be less than -1. Both D-scores and K-score also has two components. In case of D-scores, the positive of both components combined can have less than +1.5 as one of those is taken half. On the other hand, for K-scores, both components can never have maximum value at the same time so combined maximum value have +1. Since KT score does not have any negative components so value will be positive. It can have maximum value 1 when denominator is equal to $n$ and 0 when denominator is 0.
\end{proof}

\subsection{Prasatul Matrix Generation}
Prasatul matrix generation process is quite simple and straight forward. A simple algorithm called \emph{$\mathcal{L}$-Matrix Algorithm} is designed for generating prayatul matrix as shown in the Algorithm~\ref{LmatAlgo}. The algorithm takes two inputs: outcome $P$ of the primary algorithm and outcome $Q$ of the alternative algorithm. The entries of prayatul matrix is computed based on the abstraction levels as specified above and finally the algorithm returns the prasatul matrix $L$. The time complexity of the algorithm is $\mathcal{O}(N)$, where $N$ is the number trails of primary and alternative algorithms.

\begin{algorithm}
\caption{$\mathcal{L}$-Matrix Algorithm}\label{LmatAlgo}
\DontPrintSemicolon
\KwIn{$P, Q$}
\KwOut{ Prasatul Matrix $(\mathcal{L})$}
\SetKwBlock{Begin}{procedure}{end procedure}
\Begin($\text{generateMatrix} {(} P, Q {)}$)
{
 $ U\leftarrow  [P, Q]$ list arranged in non-decreasing order\; 
 $[U_b, U_\sigma, U_w] \leftarrow $ \{Best, Mean and Worst from $U$\}\;
  $C \leftarrow \{c_1, c_2, c_3, ...c_9\}$\;
  $c_i \leftarrow 0, \forall c_i \in C$\;
  \ForAll{$p_i\in P$}
  {
    \uIf{$p_i<=U_b$}
    {
        \uIf{$p_i<q_i$}
        {
            $c_1 \leftarrow c_1 + 1$\;
        }
        \uElseIf{$p_i=q_i$}
        {
            $c_4 \leftarrow c_4 + 1$\;
        }
        \uElseIf{$p_i>q_i$}
        {
            $c_7 \leftarrow c_7 + 1$\;
        }
    }
    \uElseIf{$p_i<=U_\sigma$}
    {
      \uIf{$p_i<q_i$}
        {
            $c_2 \leftarrow c_2 + 1$\;
        }
        \uElseIf{$p_i=q_i$}
        {
            $c_5 \leftarrow c_5 + 1$\;
        }
        \uElseIf{$p_i>q_i$}
        {
            $c_8 \leftarrow c_8 + 1$\;
        }
    }
    \uElseIf{$p_i<=U_w$}
    {
      \uIf{$p_i<q_i$}
        {
            $c_3 \leftarrow c_3 + 1$\;
        }
        \uElseIf{$p_i=q_i$}
        {
            $c_6 \leftarrow c_6 + 1$\;
        }
        \uElseIf{$p_i>q_i$}
        {
            $c_9 \leftarrow c_9 + 1$\;
        }
    }
  
  }\label{endfor}
  \ForAll{$i=[1,2,3],j=[1,2,3]$}
  {
    $k \gets i+j-1$\;
    $\mathcal{L}_{ij} \gets c_{k}, c_k\in C$\;
  }
  \Return{$\mathcal{L}$}
}
\end{algorithm}

\section{Performance Evaluation}
The performance of the EOAs are evaluated as follows: 1) primary algorithm compared with one or multiple alternative algorithms and 2) ranking of algorithms, each of the algorithm will get a chance to act as primary algorithm while others will act as alternative algorithm.
\subsection{Direct Comparison of Multiple Algorithm }
\label{sec3.1}
The prasatul matrix prepared for a pair of two algorithms will inherently give the one-to-one comparison of two algorithms. While D-score and K-scores obtained thereby will provide the performance of primary algorithm in comparison to alternative algorithm. To evaluate performance of a primary algorithm in comparison to multiple alternative algorithms i.e. for one-to-many comparison, multiple pairs of primary and alternatives have to be considered. If a primary algorithm has to be compared with $k$ alternatives, there will be  $k$ $\mathcal{L}$-Matrices for $k$ pairs of primary and alternatives. For each of the $k$ pairs, there will be separate sets of D-scores and K-scores indicating performance of the primary algorithm in comparison to $k$ alternatives. To get the overall performance of primary algorithm in comparison to $k$ alternatives on a specific problem, all D-score and K-score are taken into account simply by averaging as follows:
\begin{equation}
\label{eq:AD}
\begin{aligned}
    ADO=\frac{\sum_{i=1}^k{DO_i}}{k}\\ 
    ADC=\frac{\sum_{i=1}^k{DC_i}}{k}
\end{aligned}
\end{equation}
and
\begin{equation}
\label{eq:AK}
\begin{aligned}
    AKO=\frac{\sum_{i=1}^k{KO_i}}{k}\\
    AKC=\frac{\sum_{i=1}^k{KC_i}}{k}\\ 
    AKT=\frac{\sum_{i=1}^k{KT_i}}{k}
\end{aligned}
\end{equation}
where, equation (\ref{eq:AD}) shows the average DO and DC scores. Equation (\ref{eq:AK}) shows the average KO, KC and KT scores.

\subsection{Ranking with Direct Comparison Approach}
\label{sec3.2}
\subsubsection{Problem-wise Ranking}
Problem-wise ranking of EOAs for $k$ algorithms, $d$ dimensions are done in terms of D-scores and K-scores as follows:

\begin{equation}
\label{eq:PRD}
\begin{aligned}
    PDO=\frac{\sum_{i=1}^k\sum_{j=1}^d{DO_{ij}}}{k\times d}\\
    PDC=\frac{\sum_{i=1}^k\sum_{j=1}^d{DC_{ij}}}{k\times d}
\end{aligned}
\end{equation}
and
\begin{equation}
\label{eq:PRK}
\begin{aligned}
    PKO=\frac{\sum_{i=1}^k\sum_{j=1}^d{KO_{ij}}}{k\times d}\\
    PKC=\frac{\sum_{i=1}^k\sum_{j=1}^d{KC_{ij}}}{k\times d}\\
    PKT=\frac{\sum_{i=1}^k\sum_{j=1}^d{KT_{ij}}}{k\times d}
\end{aligned}
\end{equation}
where, equation (\ref{eq:PRD}) shows the problem-wise average DO and DC scores. Equation (\ref{eq:PRK}) shows the problem-wise average KO, KC and KT scores.

\subsubsection{Overall Ranking}

Overall ranking of EOAs for $k$ algorithms, $d$ dimensions, and $p$ problems are done in terms of D-scores and K-scores as follows:

\begin{equation}
\label{eq:ORD}
\begin{aligned}
    ODO=\frac{\sum_{i=1}^k\sum_{j=1}^d\sum_{l=1}^p{DO_{ijl}}}{k\times d \times p}\\
    ODC=\frac{\sum_{i=1}^k\sum_{j=1}^d\sum_{l=1}^p{DC_{ijl}}}{k\times d \times p}
\end{aligned}
\end{equation}
and
\begin{equation}
\label{eq:ORK}
\begin{aligned}
    OKO=\frac{\sum_{i=1}^k\sum_{j=1}^d\sum_{l=1}^p{KO_{ijl}}}{k\times d \times p}\\
    OKC=\frac{\sum_{i=1}^k\sum_{j=1}^d\sum_{l=1}^p{KC_{ijl}}}{k\times d \times p}\\
    OKT=\frac{\sum_{i=1}^k\sum_{j=1}^d\sum_{l=1}^p{KT_{ijl}}}{k\times d \times p}
\end{aligned}
\end{equation}
where, equation (\ref{eq:ORD}) shows the overall average DO and DC scores. The equation (\ref{eq:ORK}) shows the overall average KO, KC and KT scores.

\section{Experimental Setup}
The empirical analysis is performed with 25 benchmark functions. The well established algorithms are evaluated with the proposed methodology and verified with the existing studies. The details of empirical analysis such as benchmark functions, algorithms and experimental setup are briefed as follows.
\subsection{Benchmark Functions}
The set of 25 benchmark test problems (F1-F25) that appeared in the CEC-2005 special session on real parameter optimization~\cite{suganthan2005problem}. All these optimization problems are minimization problems, which comprises 5 unimodal and 20 multimodal functions. Out of 20 multimodal functions 7 are basic functions, 2 are expanded functions and 11 are hybrid functions. Optima of all the functions are displaced from the origin or from the previous position to ensure that the solutions can never be obtained at the center of the solution space. This displacement mechanism has made it difficult for the algorithms that have central tendency to converge towards the optima. Hybridization has added more difficulty to the problem so that algorithm unable to follow certain patterns to reach the optima. Details about the functions are as follows:

\begin{list}{\labelitemi}{\leftmargin=1em}
    \item 5 unimodal functions
    \begin{list}{-}{\leftmargin=1em}
        \item F1: Shifted Sphere Function.
        \item F2: Shifted Schwefel’s Problem 1.2.
        \item F3: Shifted Rotated High Conditioned Elliptic Function.
        \item F4: Shifted Schwefel’s Problem 1.2 with Noise.
        \item F5: Schwefel’s Problem 2.6 with Global Optimum on Bounds.
    \end{list}
    \item 20 multimodal functions
    \begin{list}{*}{\leftmargin=1em}
      \item 7 basic functions
      \begin{list}{-}{\leftmargin=1em}
          \item F6: Shifted Rosenbrock’s Function.  
          \item F7: Shifted Rotated Griewank without Bounds.
          \item F8: Shifted Rotated Ackley’s Function with Global Optimum on Bounds.
          \item F9: Shifted Rastrigin’s Function.
          \item F10: Shifted Rotated Rastrigin’s Function.
          \item F11: Shifted Rotated Weierstrass Function.
          \item F12: Schwefel’s problem 2.13.
      \end{list}
      \begin{list}{*}{\leftmargin=1em}
         \item 2 expanded functions
         \begin{list}{-}{\leftmargin=1em}
            \item F13: Expanded Extended Griewank’s plus Rosenbrock’s (EF8F2)
            \item F14: Shifted Rotated Expanded Scaffers F6.
         \end{list} 
         \item 11 hybrid functions (f15 to f25). Each of these functions has been defined through compositions of 10 out of the 14 previous functions.
      \end{list} 
    \end{list}
\end{list}

\subsection{Algorithms}
To analyze the efficacy of proposed prasatul matrix and measures, six widely applied optimization algorithms are considered, which include Ant Lion Optimizer (ALO)~\cite{mirjalili2015ant}, Grey Wolf Optimizer (GWO)~\cite{mirjalili2014grey}, Moth Flame Optimizer (MFO)~\cite{mirjalili2015moth}, Self-adaptive Differential Evolution (SaDE)~\cite{quin2005self}, Sine Cosine Algorithm (SCA)~\cite{mirjalili2016sca}, and Whale Optimization Algorithm (WOA)~\cite{mirjalili2016whale}. All these algorithms are executed in the same environment and same system configuration. 

\subsection{Implementation and System Configuration}
We consider default parameter settings for all the algorithms, except SaDE which is an adaptive version of Differential Evolution (DE) so parameter settings was not required. We consider MATLAB codes for all the algorithms that are available online. The algorithm for generating prasatul matrix as well as D-scores and K-scores, all are implemented in MATLAB R2015a and made available online\footnote{ MATLAB Source code is released under the GNU-GPLv3 license in GitHub (https://github.com/anupambis/Prasatul) as well as made available in MATLAB File Exchange linking the GitHub repository.}.  All the experiments are done on the Computer having Intel (R) Core (TM) i7-8565U CPU @ 1.80 GHz with 8 Cores, 4.6 GHz Speed, NVIDIA GeForce MX130 Graphics card, 16 GB RAM, 1TB HDD and 64-bit (AMD) Windows 10 Operating System.

\subsection{Experiment Environment}
All experiments are carried out with population size 100. For functions F1-F14, dimensions 10, 20 and 30 are considered, and for hybrid functions F15-F25, dimensions 2, 5 and 10 are considered for both quality comparison and convergence comparison. All algorithms are executed over 50 trials up to 1000 generations for observing best values. For convergence comparison, each algorithm is observed run up to 1000 generations and noted best value achieved at each generation. Functions F1–F14 are considered for quality comparison, which include 5 unimodal functions, and rest of the functions are multimodal functions. The functions F15–F25 are considered for convergence comparison, all of them are hybrid multimodal functions. However, for overall ranking all 25 functions i.e. F1-F25 are considered.

\begin{table*}[!ht]
\caption{D-scores obtained for One-to-One and One-to-Many comparison of best values of SaDE with other five algorithms}
\label{tab:Dscores}
    \centering
    \begin{tabular}{|l|c|c|c|c|c|c|c|c|c|c|c|c|}
    \hline
        ~ & \multicolumn{12}{c|}{D 10} \\ \hline
        F & \multicolumn{2}{c|}{SaDE vs ALO} & \multicolumn{2}{c|}{SaDE vs GWO} & \multicolumn{2}{c|}{SaDE vs MFO} & \multicolumn{2}{c|}{SaDE vs SCA} & \multicolumn{2}{c|}{SaDE vs WOA} & \multicolumn{2}{c|}{Average D-scores} \\ \hline
        ~ & DO & DC & DO & DC & DO & DC & DO & DC & DO & DC &ADO &ADC \\ \hline
        F1 & 0.9803 & 0.9803 & 0.9803 & 0.9803 & 0.1568 & 0.8888 & 0.9803 & 0.9803 & 0.9803 & 0.9803 &0.8156&0.9620\\ \hline
        F2 & 0.9803 & 0.9803 & 0.9803 & 0.9803 & 0.9803 & 0.9803 & 0.9803 & 0.9803 & 0.9803 & 0.9803 & 0.9803 & 0.9803\\ \hline
        F3 & 0.99 & 0.5 & 0.99 & 0.5 & 0.99 & 0.5 & 0.99 & 0.5 & 0.99 & 0.5  &0.99& 0.5\\ \hline
        F4 & 0.9803 & 0.9803 & 0.9803 & 0.9803 & 0.9803 & 0.9803 & 0.9803 & 0.9803 & 0.98039 & 0.9803 & 0.9803 & 0.9803\\ \hline
        F5 & 0.99 & 0.5 & 0.99 & 0.5 & 0 & 0 & 0.99 & 0.5 & 0.99 & 0.5  &0.792 & 0.4\\ \hline
        F6 & 0.99 & 0.5 & 0.99 & 0.5 & 0.99 & 0.5 & 0.99 & 0.5 & 0.99 & 0.5  & 0.99 & 0.5\\ \hline
        F7 & 0.0588 & 0.75 & 0.2156 & 0.9166 & 0.2549 & 0.9285 & 0.9607 & 0.98 & 0.9019 & 0.9787 &0.4784 & 0.9107\\ \hline
        F8 & 0 & 0 & 0.5783 & 0.1538 & 0 & 0 & -0.1955 & -0.1764 & 0 & 0  &0.0765 & -0.0045\\ \hline
        F9 & 0.9803 & 0.9803 & 0.9803 & 0.9803 & 0.9803 & 0.9803 & 0.9803 & 0.9803 & 0.9803 & 0.9803 & 0.9803 & 0.9803\\ \hline
        F10 & 0.4898 & 0.4705 & 0.1555 & 0.3529 & 0.99 & 0.5 & 0.99 & 0.5 & 0.99 & 0.5  & 0.7230 & 0.4647\\ \hline
        F11 & -0.2973 & -0.9166 & 0.5 & 0.5 & -0.2236 & -0.8333 & 0.99 & 0.5 & 0.4898 & 0.4705 & 0.2917 & -0.0558 \\ \hline
        F12 & 0.99 & 0.5 & 0.99 & 0.5 & 0.99 & 0.5 & 0.99 & 0.5 & 0.99 & 0.5  & 0.99 & 0.5\\ \hline
        F13 & -0.04 & 0.4583 & -0.2585 & 0.4186 & 0.4313 & 0.4888 & 0.99 & 0.5 & 0.99 & 0.5  & 0.4225 & 0.4731 \\ \hline
        F14 & -0.5906 & -0.1744 & -0.1162 & -0.8333 & 0.3229 & 0.4411 & 0.1555 & 0.3529 & 0.3229 & 0.4411 & 0.0189 &	0.0455\\ \hline
         ~ & \multicolumn{12}{c|}{D 20}\\ \hline
          F1 & 0.9803 & 0.9803 & 0.9803 & 0.9803 & 0.9803 & 0.9803 & 0.9803 & 0.9803 & 0.9803 & 0.9803 & 0.9803 & 0.9803\\ \hline
        F2 & 0.99 & 0.5 & 0.99 & 0.5 & 0.99 & 0.5 & 0.99 & 0.5 & 0.99 & 0.5  &0.99 & 0.5\\ \hline
        F3 & 0.99 & 0.5 & 0.99 & 0.5 & 0.99 & 0.5 & 0.99 & 0.5 & 0.99 & 0.5  &0.99 & 0.5\\ \hline
        F4 & 0.99 & 0.5 & 0.99 & 0.5 & 0.99 & 0.5 & 0.99 & 0.5 & 0.99 & 0.5  &0.99 & 0.5\\ \hline
        F5 & 0.99 & 0.5 & -0.2761 & -0.7142 & 0.4705 & 0.4898 & 0.99 & 0.5 & 0.99 & 0.5  &0.6328 & 0.2551\\ \hline
        F6 & 0.4509 & 0.4893 & 0.99 & 0.5 & 0.99 & 0.5 & 0.99 & 0.5 & 0.99 & 0.5  & 0.8821 & 0.4978\\ \hline
        F7 & 0.9803 & 0.9803 & 0.9803 & 0.9803 & 0 & 0 & 0.9803 & 0.9803 & 0 & 0  &0.5882 & 0.5882\\ \hline
        F8 & 0 & 0 & 0.0984 & -0.6428 & 0 & 0 & 0.1246 & -0.1951 & 0 & 0  & 0.0446 & -0.1675\\ \hline
        F9 & 1.4042 & 0.5980 & 1.4042 & 0.5980 & 1.4042 & 0.5980 & 1.4042 & 0.5980 & 1.4042 & 0.5980 & 1.4042 & 0.5980\\ \hline
        F10 & 0.1555 & 0.3529 & -0.8554 & -0.5652 & 0.2393 & 0.4117 & 0.99 & 0.5 & 0.99 & 0.5  & 0.3038 & 0.2398\\ \hline
        F11 & -0.0816 & -0.8 & 0 & 0 & -0.1627 & -0.875 & 0.99 & 0.5 & 0.1891 & 0.3823 & 0.1869 & -0.1585\\ \hline
        F12 & 0.3229 & 0.4411 & 0.99 & 0.5 & 0.99 & 0.5 & 0.99 & 0.5 & 0.99 & 0.5  & 0.8565 &	0.4882\\ \hline
        F13 & -0.9362 & -0.8863 & -0.9022 & -0.75 & 0.4607 & 0.4895 & 0.99 & 0.5 & 0.99 & 0.5  & 0.1204 & -0.0293\\ \hline
        F14 & -0.8303 & -0.6406 & 0 & 0 & 0.4898 & 0.4705 & -0.02 & 0.46 & 0.1891 & 0.3823 & -0.0342 &	0.1344\\ \hline 
        ~ & \multicolumn{12}{c|}{D 30}\\ \hline
        F1 & 0.99 & 0.5 & 0.99 & 0.5 & 0.99 & 0.5 & 0.99 & 0.5 & 0.99 & 0.5  & 0.99 & 0.5\\ \hline
        F2 & 0.99 & 0.5 & 0.99 & 0.5 & 0.99 & 0.5 & 0.99 & 0.5 & 0.99 & 0.5  & 0.99 & 0.5\\ \hline
        F3 & 0.2393 & 0.4117 & 0.99 & 0.5 & 0.99 & 0.5 & 0.99 & 0.5 & 0.99 & 0.5  & 0.8398 & 0.4823\\ \hline
        F4 & 0.99 & 0.5 & 0.99 & 0.5 & 0.99 & 0.5 & 0.99 & 0.5 & 0.99 & 0.5  &  0.99 & 0.5\\ \hline
        F5 & 0.99 & 0.5 & -0.5398 & 0.2575 & 0.4803 & 0.49 & 0.99 & 0.5 & 0.99 & 0.5  & 0.5821 & 0.4495\\ \hline
        F6 & 0.99 & 0.5 & 0.99 & 0.5 & 0.99 & 0.5 & 0.99 & 0.5 & 0.99 & 0.5  & 0.99 & 0.5\\ \hline
        F7 & 0.9803 & 0.9803 & 0.52941 & 0.9642 & 0.0392 & 0.6666 & 0.9019 & 0.9787 & 0.9607 & 0.98  & 0.6823	& 0.9140\\ \hline
        F8 & 0 & 0 & -0.0322 & -0.5 & 0 & 0 & 0.4604 & -0.1 & 0 & 0  & 0.0856 &	-0.12\\ \hline
        F9 & 0.99 & 0.5 & 0.99 & 0.5 & 0.99 & 0.5 & 0.99 & 0.5 & 0.99 & 0.5  & 0.99 & 0.5\\ \hline
        F10 & -0.5534 & -0.1111 & -0.4333 & -0.9285 & -0.03 & 0.4591 & 0.99 & 0.5 & 0.99 & 0.5  & 0.1926 & 0.0839\\ \hline
        F11 & -0.0217 & -0.5 & 0 & 0 & 0 & 0 & 0.99 & 0.5 & 0.1133 & 0.2941 & 0.2163 &	0.0588\\ \hline
        F12 & 0.3229 & 0.4411 & 0.99 & 0.5 & 0.99 & 0.5 & 0.99 & 0.5 & 0.99 & 0.5  & 0.8565 &	0.4882\\ \hline
        F13 & -0.5072 & 0.2777 & -0.3076 & -0.9230 & 0.4705 & 0.4898 & 0.99 & 0.5 & 0.99 & 0.5  & 0.3271 & 0.1689\\ \hline
        F14 & -0.3488 & -0.9375 & 0 & 0 & -0.3550 & 0.34 & 0.1555 & 0.3529 & -0.5384 & 0.0697 & -0.2173  &	-0.0349\\ \hline
    \end{tabular}
\end{table*}

\begin{table*}[!t]

\caption{K-scores obtained for One-to-One and One-to-Many comparison of best value of SaDE with other five algorithms}
\label{tab:Kscores}
    \centering
    \resizebox{1.8\columnwidth}{!}{
    \begin{tabular}{|l|c|c|c|c|c|c|c|c|c|c|c|c|c|c|c|c|c|c|}
    \hline
        ~ & \multicolumn{18}{c|}{D 10} \\ \hline
        F & \multicolumn{3}{c|}{SaDE vs ALO} & \multicolumn{3}{c|}{SaDE vs GWO} & \multicolumn{3}{c|}{SaDE vs MFO} & \multicolumn{3}{c|}{SaDE vs SCA} & \multicolumn{3}{c|}{SaDE vs WOA} & \multicolumn{3}{c|}{Average K-scores} \\ \hline
        ~ & KO & KC & KT & KO & KC & KT & KO & KC & KT & KO & KC & KT & KO & KC & KT  & AKO & AKC & AKT  \\ \hline
        F1 & 1 & 1 & 1 & 1 & 1 & 1 & 0.58 & 1 & 1 & 1 & 1 & 1 & 1 & 1 & 1  & 0.916 & 1 & 1  \\ \hline
        F2 & 1 & 1 & 1 & 1 & 1 & 1 & 1 & 1 & 1 & 1 & 1 & 1 & 1 & 1 & 1  & 1 & 1 & 1  \\ \hline
        F3 & 1 & 0.51 & 1 & 1 & 0.51 & 1 & 1 & 0.51 & 1 & 1 & 0.51 & 1 & 1 & 0.51 & 1  & 1 & 0.51 & 1  \\ \hline
        F4 & 1 & 1 & 1 & 1 & 1 & 1 & 1 & 1 & 1 & 1 & 1 & 1 & 1 & 1 & 1  & 1 & 1 & 1  \\ \hline
        F5 & 1 & 0.51 & 1 & 1 & 0.51 & 1 & -1 & -0.4 & 0 & 1 & 0.51 & 1 & 1 & 0.51 & 1  & 0.6 & 0.328 & 0.8  \\ \hline
        F6 & 1 & 0.51 & 1 & 1 & 0.51 & 1 & 1 & 0.51 & 1 & 1 & 0.51 & 1 & 1 & 0.51 & 1  & 1 & 0.51 & 1  \\ \hline
        F7 & 0.53 & 1 & 1 & 0.61 & 1 & 1 & 0.63 & 1 & 1 & 0.99 & 1 & 1 & 0.96 & 1 & 1  & 0.744 & 1 & 1  \\ \hline
        F8 & -1 & -0.85 & 0 & -0.37 & -0.33 & 0.26 & -0.97 & -0.73 & 0 & 0.35 & -0.28 & 0.36 & -1 & -0.58 & 0  & -0.598 & -0.554 & 0.124  \\ \hline
        F9 & 1 & 1 & 1 & 1 & 1 & 1 & 1 & 1 & 1 & 1 & 1 & 1 & 1 & 1 & 1  & 1 & 1 & 1  \\ \hline
        F10 & 1 & 0.48 & 0.98 & 1 & 0.36 & 0.9 & 1 & 0.51 & 1 & 1 & 0.51 & 1 & 1 & 0.51 & 1  & 1 & 0.474 & 0.976  \\ \hline
        F11 & -0.56 & -0.58 & 0 & -0.96 & -0.84 & 0.02 & 0.04 & -0.42 & 0.04 & 1 & 0.51 & 1 & 1 & 0.48 & 0.98  & 0.104 & -0.17 & 0.408  \\ \hline
        F12 & 1 & 0.51 & 1 & 1 & 0.51 & 1 & 1 & 0.51 & 1 & 1 & 0.51 & 1 & 1 & 0.51 & 1  & 1 & 0.51 & 1  \\ \hline
        F13 & 0.88 & 0.47 & 0.92 & 0.68 & 0.44 & 0.8 & 0.76 & 0.5 & 0.88 & 1 & 0.51 & 1 & 1 & 0.51 & 1  & 0.864 & 0.486 & 0.92  \\ \hline
        F14 & 0.68 & -0.07 & 0.46 & -0.8 & -0.76 & 0 & 1 & 0.45 & 0.96 & 1 & 0.36 & 0.9 & 1 & 0.45 & 0.96 & 0.576 & 0.086 & 0.656  \\ \hline
         ~ & \multicolumn{18}{c|}{D 20} \\ \hline
                 F1 & 1 & 1 & 1 & 1 & 1 & 1 & 1 & 1 & 1 & 1 & 1 & 1 & 1 & 1 & 1 & 1 & 1 & 1  \\ \hline
        F2 & 1 & 0.51 & 1 & 1 & 0.51 & 1 & 1 & 0.51 & 1 & 1 & 0.51 & 1 & 1 & 0.51 & 1 & 1 & 0.51 & 1  \\ \hline
        F3 & 1 & 0.51 & 1 & 1 & 0.51 & 1 & 1 & 0.51 & 1 & 1 & 0.51 & 1 & 1 & 0.51 & 1 & 1 & 0.51 & 1  \\ \hline
        F4 & 1 & 0.51 & 1 & 1 & 0.51 & 1 & 1 & 0.51 & 1 & 1 & 0.51 & 1 & 1 & 0.51 & 1 & 1 & 0.51 & 1  \\ \hline
        F5 & 1 & 0.51 & 1 & -0.2 & -0.09 & 0.06 & 0.92 & 0.5 & 0.96 & 1 & 0.51 & 1 & 1 & 0.51 & 1 & 0.744 & 0.388 & 0.804  \\ \hline
        F6 & 0.84 & 0.5 & 0.92 & 1 & 0.51 & 1 & 1 & 0.51 & 1 & 1 & 0.51 & 1 & 1 & 0.51 & 1 & 0.968 & 0.508 & 0.984  \\ \hline
        F7 & 1 & 1 & 1 & 1 & 1 & 1 & -1 & -1 & 0 & 1 & 1 & 1 & -1 & -1 & 0 & 0.2 & 0.2 & 0.6  \\ \hline
        F8 & -1 & -0.97 & 0 & -0.2 & -0.45 & 0.08 & -1 & -0.97 & 0 & 0.63 & -0.21 & 0.42 & -1 & -0.85 & 0 & -0.514 & -0.69 & 0.1  \\ \hline
        F9 & 1 & 0.61 & 1 & 1 & 0.61 & 1 & 1 & 0.61 & 1 & 1 & 0.61 & 1 & 1 & 0.61 & 1 & 1 & 0.61 & 1  \\ \hline
        F10 & 1 & 0.36 & 0.9 & -0.12 & 0.02 & 0.12 & 1 & 0.42 & 0.94 & 1 & 0.51 & 1 & 1 & 0.51 & 1 & 0.776 & 0.364 & 0.792  \\ \hline
        F11 & -0.84 & -0.94 & 0 & -1 & -1 & 0 & -0.72 & -0.76 & 0 & 1 & 0.51 & 1 & 1 & 0.39 & 0.92 & -0.112 & -0.36 & 0.384  \\ \hline
        F12 & 1 & 0.45 & 0.96 & 1 & 0.51 & 1 & 1 & 0.51 & 1 & 1 & 0.51 & 1 & 1 & 0.51 & 1 & 1 & 0.498 & 0.992  \\ \hline
        F13 & -0.16 & -0.1 & 0.02 & -0.4 & 0.11 & 0.04 & 0.88 & 0.5 & 0.94 & 1 & 0.51 & 1 & 1 & 0.51 & 1 & 0.464 & 0.306 & 0.6  \\ \hline
        F14 & 0.24 & -0.22 & 0.14 & -1 & -1 & 0 & 1 & 0.48 & 0.98 & 0.96 & 0.47 & 0.96 & 1 & 0.39 & 0.92 & 0.44 & 0.024 & 0.6  \\ \hline
         ~ & \multicolumn{18}{c|}{D 30} \\ \hline
               F1 & 1 & 0.51 & 1 & 1 & 0.51 & 1 & 1 & 0.51 & 1 & 1 & 0.51 & 1 & 1 & 0.51 & 1 & 1 & 0.51 & 1  \\ \hline
        F2 & 1 & 0.51 & 1 & 1 & 0.51 & 1 & 1 & 0.51 & 1 & 1 & 0.51 & 1 & 1 & 0.51 & 1 & 1 & 0.51 & 1  \\ \hline
        F3 & 1 & 0.42 & 0.94 & 1 & 0.51 & 1 & 1 & 0.51 & 1 & 1 & 0.51 & 1 & 1 & 0.51 & 1 & 1 & 0.492 & 0.988  \\ \hline
        F4 & 1 & 0.51 & 1 & 1 & 0.51 & 1 & 1 & 0.51 & 1 & 1 & 0.51 & 1 & 1 & 0.51 & 1 & 1 & 0.51 & 1  \\ \hline
        F5 & 1 & 0.51 & 1 & 0.28 & 0.35 & 0.54 & 0.96 & 0.5 & 0.98 & 1 & 0.51 & 1 & 1 & 0.51 & 1 & 0.848 & 0.476 & 0.904  \\ \hline
        F6 & 1 & 0.51 & 1 & 1 & 0.51 & 1 & 1 & 0.51 & 1 & 1 & 0.51 & 1 & 1 & 0.51 & 1 & 1 & 0.51 & 1  \\ \hline
        F7 & 1 & 1 & 1 & 0.77 & 1 & 1 & 0.52 & 1 & 1 & 0.96 & 1 & 1 & 0.99 & 1 & 1 & 0.848 & 1 & 1  \\ \hline
        F8 & -1 & -1 & 0 & -0.93 & -0.39 & 0.02 & -1 & -1 & 0 & -0.81 & -0.39 & 0.04 & -1 & -0.91 & 0 & -0.948 & -0.738 & 0.012  \\ \hline
        F9 & 1 & 0.51 & 1 & 1 & 0.51 & 1 & 1 & 0.51 & 1 & 1 & 0.51 & 1 & 1 & 0.51 & 1 & 1 & 0.51 & 1  \\ \hline
        F10 & 0.76 & -0.04 & 0.52 & -0.48 & -0.37 & 0 & 0.92 & 0.47 & 0.94 & 1 & 0.51 & 1 & 1 & 0.51 & 1 & 0.64 & 0.216 & 0.692  \\ \hline
        F11 & -0.96 & -0.85 & 0 & -1 & -1 & 0 & -1 & -0.94 & 0 & 1 & 0.51 & 1 & 1 & 0.3 & 0.86 & -0.192 & -0.396 & 0.372  \\ \hline
        F12 & 1 & 0.45 & 0.96 & 1 & 0.51 & 1 & 1 & 0.51 & 1 & 1 & 0.51 & 1 & 1 & 0.51 & 1 & 1 & 0.498 & 0.992  \\ \hline
        F13 & 0.4 & 0.35 & 0.6 & -0.52 & -0.64 & 0 & 0.92 & 0.5 & 0.96 & 1 & 0.51 & 1 & 1 & 0.51 & 1 & 0.56 & 0.246 & 0.712  \\ \hline
        F14 & -0.4 & -0.76 & 0 & -1 & -1 & 0 & 0.96 & 0.35 & 0.88 & 1 & 0.36 & 0.9 & 0.68 & 0.14 & 0.6 & 0.248 & -0.182 & 0.476  \\ \hline
    \end{tabular}
    }
\end{table*}    

\begin{table}[!ht]
    \centering
    \caption{D-scores and K-scores obtained for One-to-One and One-to-Many comparison of Convergence of ALO with other four algorithms in 10 dimensions}
    \label{tab:CDKscores}
    \resizebox{\columnwidth}{!}{
    \begin{tabular}{|@{\hspace{2pt}}l@{\hspace{2pt}}|@{\hspace{2pt}}c@{\hspace{2pt}}|@{\hspace{2pt}}c@{\hspace{2pt}}|@{\hspace{2pt}}c@{\hspace{2pt}}|@{\hspace{2pt}}c@{\hspace{2pt}}|@{\hspace{2pt}}c@{\hspace{2pt}}|@{\hspace{2pt}}c@{\hspace{2pt}}|@{\hspace{2pt}}c@{\hspace{2pt}}|@{\hspace{2pt}}c@{\hspace{2pt}}|@{\hspace{1pt}}c@{\hspace{2pt}}|@{\hspace{2pt}}c@{\hspace{1pt}}|}
    \hline
        \multirow{3}{*}{F} & \multicolumn{2}{@{\hspace{2pt}}c@{\hspace{2pt}}|}{ALO vs GWO} & \multicolumn{2}{@{\hspace{2pt}}c@{\hspace{2pt}}|}{ALO vs MFO} & \multicolumn{2}{@{\hspace{2pt}}c@{\hspace{2pt}}|}{ALO vs SCA} & \multicolumn{2}{@{\hspace{2pt}}c@{\hspace{2pt}}|}{ALO vs WOA} & \multicolumn{2}{@{\hspace{2pt}}c@{\hspace{1pt}}|}{Average Scores} \\ \hline
          &\multicolumn{8}{@{\hspace{1pt}}c@{\hspace{1pt}}|}{D-scores}& \multicolumn{2}{@{\hspace{2pt}}c@{\hspace{1pt}}|}{Average D-scores}\\ \hline
        ~ & DO & DC & DO & DC & DO & DC & DO & DC & ADO & ADC  \\ \hline
        F15 & 0.7964 & 0.3976 & 0.8207 & 0.3955 & 0.4421 & 0.3477 & 1.3435 & 0.4906 & 0.8507 & 0.4078  \\ \hline
        F16 & 0.2300 & 0.4554 & 1.0321 & 0.4313 & 0.3840 & 0.3472 & 0.3839 & 0.3457 & 0.5075 & 0.3949  \\ \hline
        F17 & 0.1876 & 0.4341 & 0.8882 & 0.4095 & 0.2591 & 0.3467 & 0.2590 & 0.3452 & 0.3984 & 0.3839  \\ \hline
        F18 & 0.1680 & 0.4597 & -0.0240 & -0.9600 & 0.4323 & 0.3542 & 0.5108 & 0.3651 & 0.2718 & 0.0547  \\ \hline
        F19 & -0.1040 & 0.3318 & -0.0240 & -0.9600 & 0.4322 & 0.3527 & 0.4335 & 0.3706 & 0.1845 & 0.0238  \\ \hline
        F20 & 0.1325 & 0.4526 & -0.0250 & -0.9615 & -0.5098 & 0.3413 & 0.3982 & 0.3537 & -0.0010 & 0.0465  \\ \hline
        F21 & 0.5738 & 0.3724 & -0.0569 & -0.9828 & 0.4798 & 0.3646 & 0.4796 & 0.3616 & 0.3691 & 0.0290  \\ \hline
        F22 & -0.6865 & -0.9004 & 0.6100 & 0.0629 & 0.1760 & 0.3507 & 0.1786 & 0.3821 & 0.0695 & -0.0262  \\ \hline
        F23 & 0.6594 & 0.3917 & -0.0490 & -0.9800 & 0.4846 & 0.3891 & 0.4832 & 0.3711 & 0.3946 & 0.0430  \\ \hline
        F24 & 0.1822 & 0.4630 & 0.5024 & 0.4216 & 0.6093 & 0.4935 & 0.5096 & 0.4860 & 0.4508 & 0.4660  \\ \hline
        F25 & -0.0037 & -0.6667 & -0.0130 & -0.9286 & 0.2472 & 0.3464 & -0.0020 & -0.6667 & 0.0571 & -0.4789  \\ \hline
         \multirow{2}{*}{}   &\multicolumn{8}{@{\hspace{1pt}}c@{\hspace{1pt}}|}{K-scores}& \multicolumn{2}{@{\hspace{1pt}}c@{\hspace{1pt}}|}{Average K-scores}\\ \hline
                 ~ & KO & KC & KO & KC & KO & KC & KO & KC & AKO & AKC  \\ \hline
        F15 & 0.922 & 0.3435 & 0.912 & 0.3345 & 1 & 0.348 & 0.798 & 0.3405 & 0.908 & 0.3416  \\ \hline
        F16 & 0.432 & 0.3125 & 0.788 & 0.31 & 1 & 0.3475 & 1 & 0.346 & 0.805 & 0.3290  \\ \hline
        F17 & 0.56 & 0.293 & 0.8 & 0.2735 & 1 & 0.347 & 1 & 0.3455 & 0.84 & 0.3148  \\ \hline
        F18 & 0.066 & 0.332 & -0.952 & -1 & 1 & 0.3545 & 0.984 & 0.3545 & 0.2745 & 0.0103  \\ \hline
        F19 & -0.336 & 0.299 & -0.952 & -1 & 1 & 0.353 & 1 & 0.371 & 0.178 & 0.0057  \\ \hline
        F20 & -0.032 & 0.3305 & -0.95 & -1 & 0.932 & 0.347 & 1 & 0.354 & 0.2375 & 0.0079  \\ \hline
        F21 & 0.98 & 0.359 & -0.886 & -1 & 1 & 0.365 & 1 & 0.362 & 0.5235 & 0.0215  \\ \hline
        F22 & -0.5 & -0.001 & -0.716 & 0.015 & 1 & 0.351 & 1 & 0.3825 & 0.196 & 0.1869  \\ \hline
        F23 & 0.964 & 0.367 & -0.902 & -1 & 1 & 0.3895 & 1 & 0.3715 & 0.5155 & 0.0320  \\ \hline
        F24 & 0.16 & 0.344 & 1 & 0.422 & 0.988 & 0.485 & 1 & 0.4865 & 0.787 & 0.4344  \\ \hline
        F25 & -0.996 & -0.3115 & -0.974 & -1 & 0.984 & 0.336 & -0.996 & -1 & -0.4955 & -0.4939  \\ \hline
    \end{tabular}
    }
\end{table}

\begin{table*}[!t]
 \caption{Problem-wise best value ranking based on D-scores and K-scores with 10, 20 and 30 dimensions combined}
 \label{tab:psqrank}
    \centering
    \resizebox{2\columnwidth}{!}{
    \begin{tabular}{|l|c|c|c|c|c|c|c|c|c|c|c|c|c|c|}
    \hline
       
         \multirow{2}{*}{Algos}&\multirow{2}{*}{F1}&\multirow{2}{*}{F2}&\multirow{2}{*}{F3}&\multirow{2}{*}{F4}&\multirow{2}{*}{F5}&\multirow{2}{*}{F6}&\multirow{2}{*}{F7}&\multirow{2}{*}{F8}&\multirow{2}{*}{F9}&\multirow{2}{*}{F10}&\multirow{2}{*}{F11}&\multirow{2}{*}{F12}&\multirow{2}{*}{F13}&\multirow{2}{*}{F14}  \\ 
         &&&&&&&&&&&&&&\\ \hline
         &\multicolumn{14}{c|}{PDO}\\ \hline
        ALO & 0.6939 & 0.7482 & 0.6774 & 0.3016 & 0.3153 & 0.7048 & 0.0337 & 0.2576 & 0.1326 & 0.1720 & 0.1860 & 0.4016 & 0.5255 & 0.2444 \\ \hline
        GWO & 0.2967 & 0.2424 & 0.1379 & 0.2806 & 0.3601 & 0.3379 & 0.4058 & -0.1254 & 0.1538 & 0.5165 & 0.4600 & 0.0985 & 0.5226 & 0.3500 \\ \hline
        MFO & 0.3932 & 0.4003 & 0.2317 & 0.6021 & 0.5983 & 0.3623 & 0.7195 & 0.2941 & 0.1903 & 0.3387 & 0.0962 & 0.0880 & 0.3901 & 0.1519 \\ \hline
        SaDE & 0.9286 & 0.9867 & 0.9399 & 0.9867 & 0.6689 & 0.9540 & 0.5830 & 0.0689 & 1.1248 & 0.4065 & 0.2316 & 0.9010 & 0.2900 & -0.0775 \\ \hline
        SCA & -0.0013 & 0.1777 & -0.0033 & 0.1545 & -0.0764 & -0.0202 & 0.2322 & 0.0388 & -0.0779 & -0.1472 & -0.0265 & -0.0235 & -0.0310 & -0.0858 \\ \hline
        WOA & 0.4606 & 0 & 0.0284 & 0 & 0.1126 & 0.4205 & 0.4969 & 0.0237 & 0.095 & 0.1654 & 0.0631 & 0.0320 & 0.0075 & 0.0672 \\ \hline
         &\multicolumn{14}{c|}{PDC}\\ \hline
               ALO & 0.3955 & 0.3990 & 0.3919 & 0.1480 & 0.0861 & 0.397 & 0.0953 & 0.3313 & 0.0849 & 0.0604 & 0.2132 & 0.3177 & 0.1826 & 0.1294 \\ \hline
        GWO & 0.1659 & 0.2033 & 0.2508 & 0.2376 & 0.2934 & 0.2048 & 0.3697 & -0.1885 & 0.2712 & 0.3160 & 0.4175 & 0.0718 & 0.3712 & 0.3578 \\ \hline
        MFO & 0.3898 & 0.2686 & 0.2373 & 0.3134 & 0.5993 & 0.2416 & 0.6043 & 0.2889 & 0.0144 & 0.1524 & 0.1348 & 0.1545 & 0.2817 & -0.0057 \\ \hline
        SaDE & 0.8141 & 0.6601 & 0.4941 & 0.6601 & 0.3682 & 0.4992 & 0.8043 & -0.0973 & 0.6928 & 0.2628 & -0.0518 & 0.4921 & 0.2042 & 0.04833 \\ \hline
        SCA & -0.0333 & -0.0055 & -0.0666 & -0.013 & 0.0295 & -0.1055 & 0.1421 & -0.0778 & -0.1347 & -0.1124 & -0.1059 & -0.1125 & -0.2133 & -0.3830 \\ \hline
        WOA & 0.2948 & 0 & -0.0884 & 0 & 0.0888 & 0.2562 & 0.2474 & -0.0042 & 0.0477 & 0.1 & 0.0425 & 0.0293 & -0.0741 & -0.0344 \\ \hline
         &\multicolumn{14}{c|}{PKO}\\ \hline
        ALO & 0.472 & 0.5573 & 0.5866 & 0.1253 & 0.1733 & 0.4826 & -0.3813 & 0.9493 & 0.0433 & 0.2746 & 0.4693 & 0.464 & 0.3253 & 0.3226 \\ \hline
        GWO & -0.3306 & 0.0453 & 0.1893 & 0.32 & 0.4213 & -0.336 & 0.1406 & -0.336 & 0.552 & 0.6826 & 0.9946 & 0.1173 & 0.7333 & 0.984  \\ \hline
        MFO & -0.124 & -0.0373 & -0.232 & 0.0933 & 0.2746 & -0.2106 & 0.636 & 0.6193 & 0.006 & -0.168 & 0.2026 & -0.1226 & -0.1626 & -0.6773 \\ \hline
        SaDE & 0.972 & 1 & 1 & 1 & 0.7306 & 0.9893 & 0.5973 & -0.6866 & 1 & 0.8053 & -0.0666 & 1 & 0.6293 & 0.4213 \\ \hline
        SCA & -0.9973 & -0.5653 & -0.9946 & -0.5386 & -0.712 & -0.9786 & -0.3486 & -0.7473 & -0.8453 & -0.7226 & -0.9626 & -0.968 & -0.9466 & -0.605 \\ \hline
        WOA & 0.064 & -1 & -0.5493 & -1 & -0.888 & 0.0533 & -0.0213 & 0.22 & -0.752 & -0.872 & -0.6373 & -0.4906 & -0.5786 & -0.4453 \\ \hline
         &\multicolumn{14}{c|}{PKC}\\ \hline
        ALO & 0.2166 & 0.3073 & 0.266 & 0.0566 & -0.0593 & 0.386 & -0.2393 & 0.3513 & -0.1653 & 0.0533 & -0.034 & 0.214 & 0.1433 & -0.182  \\ \hline
        GWO & -0.0426 & 0.014 & 0.0453 & 0.1493 & 0.2186 & 0.132 & 0.176 & -0.64 & 0.0906 & 0.1773 & 0.4266 & 0.0526 & 0.236 & 0.356  \\ \hline
        MFO & -0.0046 & -0.0713 & -0.0013 & -0.042 & 0.2106 & 0.0193 & 0.6706 & 0.1573 & -0.172 & -0.176 & -0.106 & -0.0966 & 0.0266 & -0.6006 \\ \hline
        SaDE & 0.8366 & 0.6733 & 0.504 & 0.6733 & 0.3973 & 0.5093 & 0.7333 & -0.6606 & 0.7066 & 0.3513 & -0.3086 & 0.502 & 0.346 & -0.024  \\ \hline
        SCA & -0.994 & -0.65 & -0.782 & -0.584 & -0.6993 & -0.738 & -0.3393 & -0.6746 & -0.82 & -0.784 & -0.936 & -0.842 & -0.92 & -0.642  \\ \hline
        WOA & 0.072 & -0.9 & -0.386 & -0.916 & -0.5466 & 0.2313 & -0.0933 & -0.258 & -0.694 & -0.728 & -0.652 & -0.186 & -0.512 & -0.5493 \\ \hline
         &\multicolumn{14}{c|}{PKT}\\ \hline
        ALO & 0.736 & 0.7786 & 0.788 & 0.5013 & 0.512 & 0.7226 & 0.3253 & 0.872 & 0.4333 & 0.54 & 0.5946 & 0.6773 & 0.596 & 0.4853 \\ \hline
        GWO & 0.328 & 0.4866 & 0.5373 & 0.6013 & 0.6426 & 0.3293 & 0.5906 & 0.1466 & 0.696 & 0.748 & 0.9426 & 0.4986 & 0.788 & 0.8973 \\ \hline
        MFO & 0.432 & 0.4493 & 0.3546 & 0.5053 & 0.6186 & 0.3786 & 0.8413 & 0.6986 & 0.42 & 0.3773 & 0.4826 & 0.364 & 0.3906 & 0.0973 \\ \hline
        SaDE & 1 & 1 & 0.996 & 1 & 0.836 & 0.9946 & 0.8666 & 0.0786 & 1 & 0.82 & 0.388 & 0.9946 & 0.744 & 0.5773 \\ \hline
        SCA & 0 & 0.2 & 0 & 0.2 & 0.112 & 0 & 0.3133 & 0.0693 & 0.0333 & 0.0653 & 0 & 0 & 0 & 0.0666 \\ \hline
        WOA & 0.532 & 0 & 0.176 & 0 & 0.056 & 0.5253 & 0.488 & 0.432 & 0.0933 & 0.064 & 0.1453 & 0.2146 & 0.152 & 0.1546 \\ \hline
    \end{tabular}
    }
\end{table*}

\begin{table*}[t]
    \centering
    \caption{Problem-wise convergence ranking  based on KT score with 2, 5 and 10 dimensions combined}
    \label{tab:pconrank}
    \begin{tabular}{|l|c|c|c|c|c|c|c|c|c|c|c|}
    \hline
        & \multicolumn{11}{c|}{\multirow{2}{*}{PKT}}\\
         &\multicolumn{11}{c|}{}\\ \hline
        Algorithm & F15 & F16 & F17 & F18 & F19 & F20 & F21 & F22 & F23 & F24 & F25  \\ \hline
        
        ALO & 0.6110 & 0.6731 & 0.6478 & 0.4248 & 0.4482 & 0.5994 & 0.7578 & 0.4666 & 0.4909 & 0.8048 & 0.6153  \\ \hline
        GWO & 0.3531 & 0.5968 & 0.3857 & 0.3500 & 0.4957 & 0.1788 & 0.2778 & 0.5734 & 0.3920 & 0.6357 & 0.2607  \\ \hline
        MFO & 0.6379 & 0.6646 & 0.6130 & 0.6294 & 0.6000 & 0.6753 & 0.6337 & 0.2896 & 0.6028 & 0.4769 & 0.6024  \\ \hline
        SCA & 0.3249 & 0.3174 & 0.4443 & 0.3878 & 0.3588 & 0.5074 & 0.3717 & 0.5738 & 0.5378 & 0.1458 & 0.3227  \\ \hline
        WOA & 0.3154 & 0.0561 & 0.1610 & 0.3764 & 0.3432 & 0.3455 & 0.2118 & 0.2494 & 0.2247 & 0.3937 & 0.5212  \\ \hline
    \end{tabular}
\end{table*}

\begin{table}[t]
    \caption{Overall best value ranking based on D-scores and K-scores}
    \label{tab:osqrank}
    \centering
    \begin{tabular}{|c|c|c|c|c|c|c|}
    \hline
    
    \multirow{3}{*}{Algorithms}& \multicolumn{5}{c|}{\multirow{2}{*}{Overall Best Value Ranking F1-F14}}\\
        &\multicolumn{5}{c|}{}\\
    \cline{2-6}
        & ODO & ODC& OKO & OKC & OKT\\ 
     \hline
 ALO  & 0.3853 & 0.2309 & 0.3474 & 0.0939 & 0.6116 \\ \hline
        GWO  & 0.2884 & 0.2387 & 0.2984 & 0.0994 & 0.5880 \\ \hline
        MFO  & 0.3469 & 0.2625 & 0.0069 & -0.0132 & 0.4579 \\ \hline
        SaDE  & 0.6424 & 0.4179 & 0.6708 & 0.3742 & 0.8068 \\ \hline
        SCA  & 0.0078 & -0.0852 & -0.7809 & -0.7432 & 0.0757 \\ \hline
        WOA & 0.1409 & 0.0647 & -0.4310 & -0.4546 & 0.2432 \\ 
    \hline    
    &\multicolumn{5}{c|}{Overall Best Value Ranking F15-F25}  \\ \hline
        ALO  & 0.3495 & 0.1527 & 0.291 & -0.0472 & 0.4892  \\ \hline
        GWO  & 0.2434 & 0.0452 & 0.1437 & -0.0817 & 0.4224  \\ \hline
        MFO  & 0.2837 & 0.1639 & 0.2358 & -0.0734 & 0.472 \\ \hline
        SCA  & -0.0279 & 0.013 & -0.1645 & -0.1623 & 0.3035  \\ \hline
        WOA  & 0.3104 & 0.1816 & -0.2939 & -0.2680 & 0.28 \\ \hline
    \end{tabular}

\end{table}

\begin{table}[t]
    \caption{Overall convergence ranking based on D-scores and K-scores}
    \label{tab:oconrank}
    \centering
    \begin{tabular}{|c|c|c|c|c|c|c|}
    \hline
     \multirow{3}{*}{Algorithms}& \multicolumn{5}{c|}{\multirow{2}{*}{Overall Best Value Ranking F1-F14}}\\
        &\multicolumn{5}{c|}{}\\
    \cline{2-6}
        & ODO & ODC& OKO & OKC & OKT\\ 
     \hline
        ALO  & 0.2454 & 0.1427 & 0.3657 & 0.1679 & 0.6126  \\ \hline
        GWO  & -0.0803 & -0.0271 & 0.2046 & 0.0906 & 0.5224  \\ \hline
        MFO  & 0.0142 & 0.1372 & 0.0325 & 0.1819 & 0.4741  \\ \hline
        SaDE  & 0.2455 & 0.2642 & 0.5956 & 0.1990 & 0.7371  \\ \hline
        SCA  & 0.3693 & 0.1064 & -0.8321 & -0.4795 & 0.0706  \\ \hline
        WOA  & -0.1002 & -0.3047 & -0.3661 & -0.3229 & 0.2648  \\ \hline
    &\multicolumn{5}{c|}{Overall Convergence Ranking F15-F25}  \\ \hline
        ALO  & 0.3186 & 0.0707 & 0.2934 & 0.0552 & 0.5945  \\ \hline
        GWO  & 0.0731 & -0.1047 & -0.057 & -0.1975 & 0.4091  \\ \hline
        MFO  & 0.3921 & 0.2616 & 0.243 & 0.1557 & 0.5841  \\ \hline
        SCA  & 0.5898 & 0.0641 & -0.1478 & -0.2147 & 0.3902  \\ \hline
        WOA  & 0.0374 & -0.2167 & -0.3197 & -0.3484 & 0.2908  \\ \hline
    \end{tabular}

\end{table}

\begin{table}[!ht]
    \centering
    \caption{Wilcoxon rank-sum test p-values with $5\%$ significance for F1–F14 with 30 dimensions for SaDE paired with other five algorithms}
    \label{tab:wrank}
    \resizebox{\columnwidth}{!}{
    \begin{tabular}{|l|l|l|l|l|l|l|l|l|l|l|}
    \hline
         \multirow{2}{*}{F}&\multicolumn{2}{c|}{SaDE vs ALO} & \multicolumn{2}{c|}{SaDE vs GWO} & \multicolumn{2}{c|}{SaDE vs MFO} & \multicolumn{2}{c|}{SaDE vs SCA} & \multicolumn{2}{c|}{SaDE vs WOA}   \\ \cline{2-11}
        ~ & p-value & T & p-value & T & p-value & T & p-value & T & p-value & T  \\ \hline
        F1 & 7.00E-18 & + & 7.00E-18 & + & 7.00E-18 & + & 7.00E-18 & + & 7.00E-18 & +  \\ \hline
        F2 & 2.00E-17 & + & 7.00E-18 & + & 7.00E-18 & + & 7.00E-18 & + & 7.00E-18 & +  \\ \hline
        F3 & 5.00E-15 & + & 7.00E-18 & + & 8.00E-18 & + & 7.00E-18 & + & 7.00E-18 & +  \\ \hline
        F4 & 7.00E-18 & + & 8.00E-18 & + & 2.00E-17 & + & 7.00E-18 & + & 7.00E-18 & +  \\ \hline
        F5 & 7.00E-18 & + & 2.00E-02 & + & 5.00E-16 & + & 7.00E-18 & + & 7.00E-18 & +  \\ \hline
        F6 & 3.00E-14 & + & 7.00E-18 & + & 7.00E-18 & + & 7.00E-18 & + & 7.00E-18 & +  \\ \hline
        F7 & 3.00E-20 & + & 1.00E-09 & + & 2.00E-01 & - & 3.00E-18 & + & 7.00E-20 & +  \\ \hline
        F8 & 7.00E-18 & + & 7.00E-02 & - & 7.00E-18 & + & 2.00E-01 & - & 3.00E-15 & +  \\ \hline
        F9 & 7.00E-18 & + & 7.00E-18 & + & 7.00E-18 & + & 7.00E-18 & + & 7.00E-18 & +  \\ \hline
        F10 & 1.00E-02 & + & 1.00E-03 & + & 3.00E-10 & + & 7.00E-18 & + & 7.00E-18 & +  \\ \hline
        F11 & 5.00E-11 & + & 7.00E-18 & + & 3.00E-12 & + & 7.00E-18 & + & 9.00E-11 & +  \\ \hline
        F12 & 6.00E-12 & + & 8.00E-18 & + & 7.00E-18 & + & 7.00E-18 & + & 7.00E-18 & +  \\ \hline
        F13 & 2.00E-02 & + & 1.00E-05 & + & 4.00E-14 & + & 7.00E-18 & + & 7.00E-18 & +  \\ \hline
        F14 & 2.00E-03 & + & 7.00E-18 & + & 3.00E-09 & + & 8.00E-13 & + & 3.00E-05 & +  \\ \hline
    \end{tabular}
    }
\end{table}

\begin{table}[!ht]
    \centering
    \caption{Wilcoxon rank-sum test p-values with $5\%$ significance for F15–F25 with 10 dimensions for ALO paired with other four algorithms}
    \label{tab:wrank2}
    \resizebox{\columnwidth}{!}{
    \begin{tabular}{|l|l|l|l|l|l|l|l|l|}
    \hline
         \multirow{2}{*}{F}&\multicolumn{2}{c|}{ALO vs GWO} & \multicolumn{2}{c|}{ALO vs MFO} & \multicolumn{2}{c|}{ALO vs SCA} & \multicolumn{2}{c|}{ALO vs WOA}   \\ \cline{2-9}
        ~ & p-value & T & p-value & T & p-value & T & p-value & T  \\ \hline
        F15 & 4.00E-02 & + & 2.00E-02 & + & 3.00E-01 & - & 8.00E-04 & +  \\ \hline
        F16 & 9.00E-03 & + & 1.00E-01 & - & 2.00E-16 & + & 2.00E-14 & +  \\ \hline
        F17 & 2.00E-02 & + & 6.00E-02 & - & 3.00E-17 & + & 5.00E-17 & +  \\ \hline
        F18 & 8.00E-01 & - & 2.00E-05 & + & 9.00E-01 & - & 9.00E-07 & +  \\ \hline
        F19 & 9.00E-01 & - & 1.00E-07 & + & 6.00E-01 & - & 3.00E-10 & +  \\ \hline
        F20 & 9.00E-01 & - & 1.00E-04 & + & 9.00E-02 & - & 3.00E-07 & +  \\ \hline
        F21 & 2.00E-01 & - & 8.00E-02 & - & 2.00E-01 & - & 2.00E-07 & +  \\ \hline
        F22 & 8.00E-01 & - & 7.00E-01 & - & 9.00E-09 & + & 2.00E-12 & +  \\ \hline
        F23 & 5.00E-02 & + & 4.00E-04 & + & 8.00E-01 & - & 1.00E-06 & +  \\ \hline
        F24 & 4.00E-04 & + & 8.00E-07 & + & 5.00E-15 & + & 8.00E-14 & +  \\ \hline
        F25 & 1.00E-09 & + & 7.00E-18 & + & 1.00E-13 & + & 1.00E-08 & +  \\ \hline
    \end{tabular}
    }
\end{table}

\section{Result Analysis}
In general, the performance of evolutionary optimization algorithms are compared on two grounds: 1) solution quality i.e. how optimal the solutions are and 2) convergence i.e. how fast the algorithm moves towards the optima. With the proposed prasatul matrix, algorithms are measured  as well as ranked on both the grounds. Also, verified the outcomes of prasatul matrix $\mathcal{L}$  with that of popularly used Wilcoxon paired rank-sum test results.  The interpretation of results is very simple, the rule-of-thumbs for end-users is highly positive (+ve) values imply good performance and negative (-ve) implies bad performance of primary algorithm.

\subsection{Solution Quality Comparison}
Best values obtained for 50 trials of each of the six algorithms are considered. Algorithms are paired (as explained in Section~\ref{sec3.1}) to generate the prasatul matrix $\mathcal{L}$  and to obtain D-scores and K-scores thereby. D-scores and K-scores obtained with all pairs of SaDE and other five algorithms with dimensions 10, 20 and 30 are presented in Table~\ref{tab:Dscores} and Table~\ref{tab:Kscores} respectively. It is clear for D-scores that SaDE performs better in most of the functions across all the three dimensions as indicated by +ve high values. However, SaDE performs poor in F7, F8, F11, F13 and F14 in some cases as indicated by low and -ve values. In F13, though DO score is -ve but DC scores are mostly +ve, which indicates that despite most of the solutions are worst in terms of Optimality but better in comparison to others. ADO as well as ADC scores also indicate the same. Similarly, K-scores also indicate that mostly SaDE performs better in terms of solution quality i.e. better both in terms of Optimality and Comparability in comparison to other five algorithms across all three dimensions.

\subsection{Convergence Comparison}
For convergence comparison using proposed direct comparison approach, best values obtained in each of the 1000 iterations over 50 trials of each of the six algorithms are averaged iteration-wise. Thereafter, followed the direct comparison to obtain prasatul matrices and respective  D-scores and K-scores. D-scores and K-scores obtained with all pairs of ALO and other four algorithms with 10 dimensions are presented in Table~\ref{tab:CDKscores}. SaDE could not be executed on F23 so it is not considered for convergence comparison. As indicated by both D-scores and K-score, mostly ALO performs better than others in terms of convergence. However, in some functions such as F19, F22 and F25, convergence of ALO is poor compared to others. Also, ALO performs worst in comparison to MFO in most of the functions as indicated by -ve values of both D-scores as well as K-scores. Though, ADO, ADO, AKO and AKC scores indicate that convergence of ALO is better than others but the scores are quite low.

\subsection{Solution Quality and Convergence Ranking}

Ranking with direction comparison approach is done in two aspects. First, the problem-wise ranking, where algorithms are ranked on basis of specific problems taking into account the different dimensions together. Second, the overall ranking, where different problems as well as their dimensionalities are considered together to get overall performance of an algorithm in comparison to other alternatives.  Both kinds of ranking are done on the grounds of both solution quality and algorithm convergence.

\subsubsection{Problem-wise Ranking}
Problem-wise ranking on F1-F14 in terms of solution quality of the algorithms is presented in Table~\ref{tab:psqrank}. As indicated by the PDO, PDC, PKO, PKO, PKC, and PKT scores, clearly SaDE is best performing algorithm in most the functions and produces good quality solutions. Interestingly, ALO and MFO seems to be strong competitors but mostly ALO ranked second after SaDE. SCA is worst performing algorithm and ranked last in almost all functions in terms of solution quality. Important to note that PKT is enough for determining problem-wise ranking of algorithms though because it is comprises both Optimality and Comparability together as it comprises KT scores. Therefore, only PKT is considered for convergence. Problem-wise ranking on F15-F25 in terms of convergence of the algorithms is presented in Table~\ref{tab:pconrank}. As indicated by PKT scores, MFO is the best performing algorithm in terms of convergence in most of the functions. Whereas, performance of WOA in terms of convergence is worst in most of the functions.

\subsection{Overall Ranking}
Overall ranking in terms of solution quality of the algorithms is presented in Table~\ref{tab:osqrank}. Echoing with the problem-wise ranking, SaDE stands as the best performing algorithm in terms of solution quality for F1-F14. SaDE is best scores in all five overall ranking scores. SCA remains the worst performing, while ALO is second best. However, for F15-F25, ALO is ranked as best and  WOA as the worst performing algorithm in terms of solution quality. 

Overall ranking in terms of convergence of the algorithms is presented in Table~\ref{tab:oconrank}. Like solution quality, SaDE is best performing algorithm in terms of convergence as well for F1-F14. However, as already noted in case of problem-wise ranking, SCA is worst and ALO is second best in terms of convergence as well for F1-F14. Though, ALO is ranked best for F15-F25 in problem-wise ranking, in overall ranking actually MFO is ranked as best. This is because, though ALO is better performing algorithm than MFO in most of the functions but MFO performed way better in some functions. Moreover, during one-to-one and one-to-many convergence comparison, already noted that ALO and MFO are strong competitors. Like problem-wise ranking, performance of WOA seems to be worst in terms of convergence in overall ranking as well.   

\subsection{Wilcoxon Rank-sum Test}
The significant difference in solutions obtained with SaDE algorithm compared to other algorithms are analyzed using the Wilcoxon rank-sum statistical test with a $5\%$ accuracy. The Wilcoxon rank-sum test p-values with a $5\%$ significance level are shown in Table~\ref{tab:wrank} and Table~\ref{tab:wrank2}. The '+' and '-' marks in Table~\ref{tab:wrank} indicate a significant +ve or -ve difference between the algorithms.  According to p-values, statistically significant differences can be seen in almost every function. However, no significant differences can be seen in Table~\ref{tab:wrank2} as indicated by p-values and T, except for WOA. Already noted in previous Section, WOA as worst performing algorithm in terms of convergence. The  Wilcoxon test reaffirms SaDE is best performing algorithm in terms of quality and WOA worst in terms of convergence as noted in the direct comparison approach results presented in Tables (1-7). The results of direct comparison approach is aligned with Wilcoxon test and in addition it gives more insights about the solutions produced by the algorithms in comparison to others both in the grounds of solution quality and convergence.



\section{Conclusion}
A direct comparison approach is proposed to evaluate the performance of EOAs both in the grounds of solution quality and convergence, where solutions obtained with the EOAs are compared directly. Introduced a direct comparison matrix called \emph{Prasatul Matrix} to record the direct comparison outcomes of solutions produced by two algorithms. The prasatul matrix has three levels of abstractions both in rows (Win, Tie \& Lose) and columns (Best, Average \& Worst). Designed five different performance measures as well as two ranking schemes based on the prasatul matrix, which can be used to evaluate algorithms both in terms of solution quality as well as convergence.  

Efficacy of the proposed direct comparison approach is analyzed with six different EOAs on 25 benchmark functions. One-to-one as well as one-to-many comparison results are analyzed and demonstrated how solution quality and convergence of EOAs can be analyzed with newly designed performance measures. Presented the results of problem-wise ranking and overall ranking for both solution quality and convergence. The outcomes of proposed approach is also verified with the outcomes of Wilcoxon rank-sum statistical test. Results indicate that the outcomes of proposed approach not only aligned with Wilcoxon rank-sum statistical test outcomes but also gives more insights about the performance of algorithms, which could not be done earlier. Most importantly, with the proposed approach, algorithms can also be ranked on the grounds of convergence.



  
\section*{Acknowledgment}
This work is supported by the Science and Engineering Board (SERB), Department of Science and Technology (DST) of the Government of India under Grant No. EEQ/2019/000657.

\ifCLASSOPTIONcaptionsoff
  \newpage
\fi

\bibliographystyle{IEEEtran}
\bibliography{myref}

\begin{IEEEbiography}[{\includegraphics[width=.9in,height=1.2in,clip,keepaspectratio]{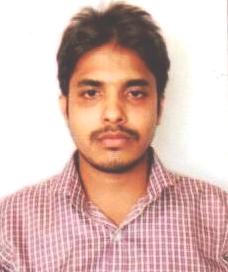}}]{Dr. Anupam Biswas}
  is currently working as an Assistant Professor in the Department of Computer Science and Engineering, National Institute of Technology Silchar, Assam, India. He has received Ph.D. degree in computer science and engineering from Indian Institute of Technology (BHU), Varanasi, India in 2017. He has received M. Tech. Degree in computer science and engineering from Motilal Nehru National Institute of Technology Allahabad, Prayagraj, India in 2013 and B. E.  degree in computer science and engineering from Jorhat Engineering College, Jorhat, Assam in 2011. He has published several research papers in transactions, reputed international journals, conference and book chapters. His research interests include Evolutionary computation, Machine learning, Social networks, Computational music, and Information retrieval. He has five granted patents, out of which four are Germany patents and one South African patent. He is the Principal Investigator of four on-going DST-SERB sponsored research projects in the domain of machine learning and evolutionary computation. He has served as Program Chair of International Conference on Big Data, Machine Learning and Applications (BigDML 2019) and Publicity Chair of BigDML 2021. He has served as General Chair of 25th International Symposium Frontiers of Research in Speech and Music (FRSM 2020) and co-edited the proceedings of FRSM 2020 published as book volume in Springer AISC Series. He has edited five books that are published by various series of Springer. Also edited a book with Advances in Computers book Series of Elsevier.
\end{IEEEbiography}

\vspace{11pt}

\vfill

\end{document}